\pdfoutput=1

\documentclass[11pt]{article}

\usepackage{emnlp2021}

\usepackage{times}
\usepackage{latexsym}

\usepackage{microtype}
\usepackage[T5,T1]{fontenc}

\usepackage{caption}
\usepackage{subcaption}
\usepackage{bbm}
\usepackage{enumerate}
\usepackage{xspace}
\usepackage{amsmath}
\usepackage{tikz}
\usetikzlibrary{bayesnet}
\usepackage{amsthm}
\usepackage{amssymb}
\usepackage{tagpdf}
\tagpdfsetup{uncompress, %
             activate-mc}

\DeclareTextSymbolDefault{\ohorn}{T5}
\DeclareTextSymbolDefault{\uhorn}{T5}

\usepackage{graphicx} 
\graphicspath{{./figs/}}

\definecolor{codec0}{HTML}{369275}
\definecolor{codec1}{HTML}{cc5d32}
\definecolor{codec2}{HTML}{5d709b}
\definecolor{codec3}{HTML}{c76aa3}
\definecolor{codec4}{HTML}{76a824}
\definecolor{codec5}{HTML}{bf990f}
\definecolor{codec6}{HTML}{e5c494}
\definecolor{codec7}{HTML}{b3b3b3}

\newcommand{\citeposs}[1]{\citeauthor{#1}'s (\citeyear{#1})}

\usepackage{xfrac}

\newcommand{\poisson}{\mathrm{Pois}}
\newcommand{\categorical}{\mathrm{Cat}}
\newcommand{\dirichlet}{\mathrm{Dir}}

\DeclareMathOperator*{\argmax}{arg\,max}

\newcommand{\R}{\mathbbm{R}}
\newcommand{\MI}{\mathrm{I}}

\newcommand{\ent}{\mathrm{H}}
\newcommand{\KL}{\mathrm{KL}}

\newcommand{\softmax}{\mathrm{softmax}}

\newcommand{\vs}{\mathbf{s}}
\newcommand{\vx}{x}
\newcommand{\vy}{y}

\newcommand{\vr}{r}
\newcommand{\vt}{t}

\newcommand{\vd}{\mathbf{d}}
\newcommand{\vphi}{{\boldsymbol \phi}}
\newcommand{\vpsi}{{\boldsymbol \psi}}
\newcommand{\vdn}{\vd_N}
\newcommand{\defn}[1]{\textbf{#1}}

\newcommand{\calT}{\mathcal{T}}
\newcommand{\calX}{\mathcal{X}}
\newcommand{\calV}{\mathcal{V}}

\newcommand{\calY}{\mathcal{Y}}
\newcommand{\calZ}{\mathcal{Z}}
\newcommand{\calD}{\mathbf{D}}
\newcommand{\calDn}{\mathbf{D}_N}

\newcommand{\vtheta}{{\boldsymbol \theta}}
\newcommand{\valpha}{{\boldsymbol \alpha}}

\newcommand{\ptheta}{p_{\vtheta}}
\newcommand{\enttheta}{\mathrm{H}_{\vtheta}}
\newcommand{\MItheta}{\mathrm{I}_{\vtheta}}
\newcommand{\MIcalV}{\mathrm{I}_{\calV}}

\newcommand{\truetheta}{\hat{\vtheta}}

\newcommand{\defeq}[0]{\mathrel{\stackrel{\textnormal{\tiny def}}{=}}}

\newtheorem{thm}{Theorem}
\newtheorem{defnt}{Definition}

\usepackage{cleveref}
\crefname{section}{\S}{\S\S}
\Crefname{section}{\S}{\S\S}
\crefname{table}{Tab.}{}
\crefname{figure}{Fig.}{Figs.}
\crefname{algorithm}{Algorithm}{}
\crefname{algorithm}{Algorithm}{}
\crefname{line}{Line}{}
\crefname{appendix}{App.}{}
\crefformat{section}{\S#2#1#3}
\crefname{thm}{Theorem}{Theorems}
\crefname{cor}{Corollary}{}
\crefname{prop}{Proposition}{}
\crefname{def}{Definition}{}
\crefname{defnt}{Definition}{}

\title{A Bayesian Framework for Information-Theoretic Probing}

\usepackage{tipa}
\usepackage{emoji}

\newcommand{\football}{\emoji[openmoji]{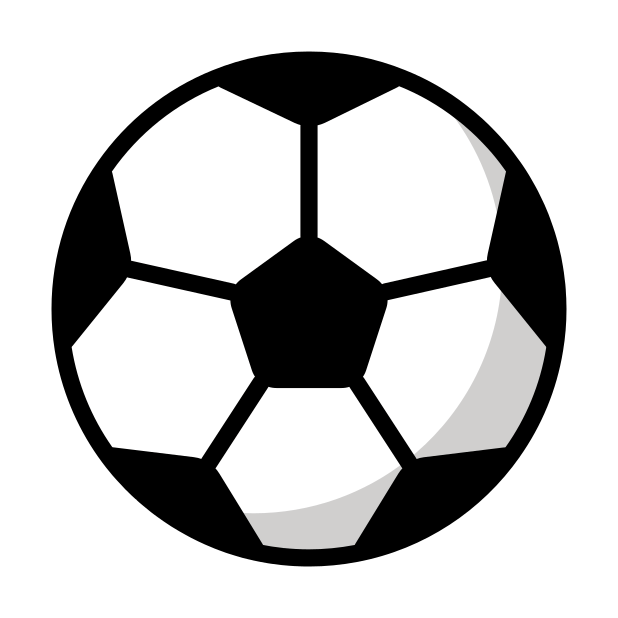}}
\newcommand{\americanfootball}{\emoji[openmoji]{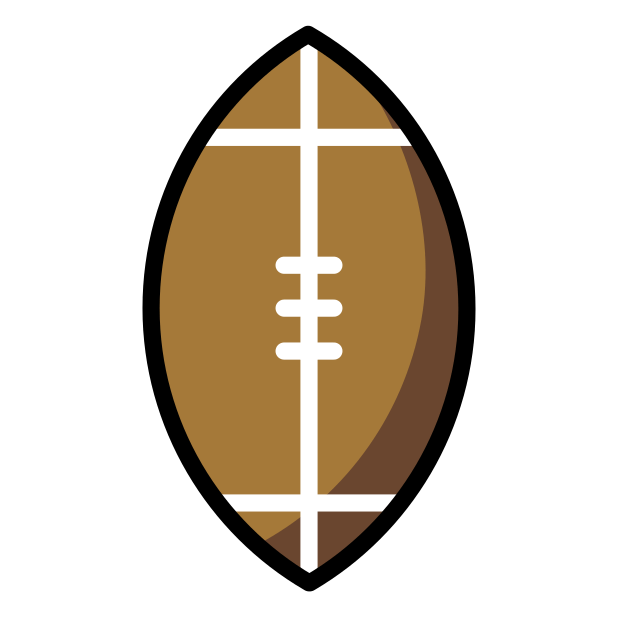}}

\newcommand{\emo}[1]{\raise1.0ex\hbox{\normalfont#1}}

\newcommand{\ucambridge}{\emo{\football}}
\newcommand{\ethz}{\emo{\americanfootball}}

\author{%
Tiago Pimentel\ucambridge{} \\
  \ucambridge{}University of Cambridge \phantom{\ucambridge{}} \\
  \texttt{\href{mailto:tp472@cam.ac.uk}{tp472@cam.ac.uk}}
  \And
  \phantom{\ucambridge{}}Ryan Cotterell\ucambridge{}\ethz{} \\
  \ethz{}ETH Z\"{u}rich \phantom{\ethz{}} \\
  \texttt{\href{mailto:ryan.cotterell@inf.ethz.ch}{ryan.cotterell@inf.ethz.ch}}
}

\date{}

\begin{document}
\maketitle

\begin{abstract}
\citet{pimentel2020information} recently analysed probing from an information-theoretic perspective. 
They argue that probing should be seen as approximating a mutual information.
This led to the rather unintuitive conclusion that representations encode exactly the same information about a target task as the original sentences.
The mutual information, however, assumes the true probability distribution of a pair of random variables is known, leading to unintuitive results in settings where it is not.
This paper proposes a new framework to measure what we term \defn{Bayesian mutual information}, which analyses information from the perspective of Bayesian agents---allowing for more intuitive findings in scenarios with finite data.
For instance, under Bayesian MI we have that data can add information, processing can help, and information can hurt, which makes it more intuitive for machine learning applications.
Finally, we apply our framework to probing where we believe Bayesian mutual information naturally operationalises \emph{ease of extraction} by explicitly limiting the available background knowledge to solve a task.
\end{abstract}

\section{Introduction}

\newcite{pimentel2020information} recently undertook an information-theoretic analysis of probing. 
They argue that probing may be viewed as approximating the mutual information between a linguistic property (e.g., part-of-speech tags) %
and a contextual representation (e.g., BERT). %
Counter-intuitively, however, due to the data-processing inequality, contextual representations contain exactly the same information about any task as the original sentence, under mild conditions. 
When viewed under this lens, the goal of probing is not inherently clear.
One limitation of \citeauthor{pimentel2020information}'s analysis is that it focuses on the \defn{mutual information} (MI)---%
to be of practical application, their argument requires that a probe matches the \emph{true distribution} according to which 
the data were generated in the limit of \emph{finite training data}. 
In contrast, our paper formulates an information-theoretic framework that is compatible both with model misspecification and the finite data assumption.\looseness=-1

\begin{figure}
    \centering
\tagmcbegin{tag=Figure,alttext={An example of Bayesian MI on dependency arc labelling in English. Two smoothed curves show both the unconditional and conditional Bayesian entropies, which converge to the true entropies in the limit of infinite data. The difference between both Bayesian entropies is the Bayesian MI.}}
    \includegraphics[width=\columnwidth]{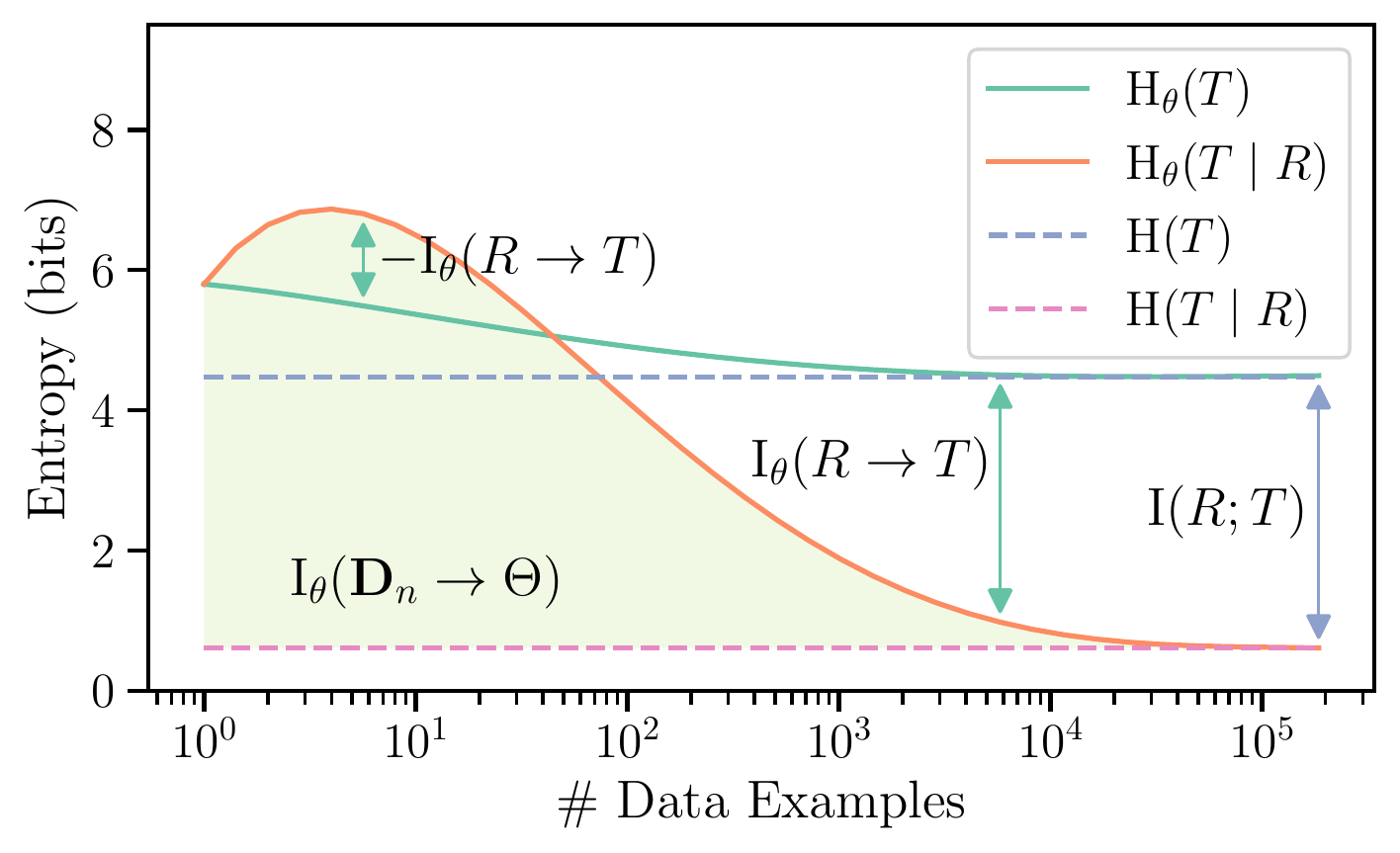}
    \vspace{-20pt}
    \caption{A smoothed example of Bayesian MI on dependency arc labelling in English.}
    \label{fig:plot_explanation}
\tagmcend
\end{figure}

In his seminal work, \citet{shannon1948mathematical} occupied himself with the \emph{limit} of communication.
Indeed, mutual information can be described as the theoretical limit (or upper-bound) of how much information can be extracted from one random variable about another. 
However, this limit is only achievable when one has full knowledge of these random variables, including the true probability distribution according to which they are distributed. 
In practice, we will not have access to such information and it may be difficult to approximate. 
It follows that any system with imperfect knowledge of the random variable's true distribution will only be able to extract a subset of this information. 

With this in mind, we propose and motivate an agent-based framework for measuring information.
We term our quantity \defn{Bayesian mutual information} and show that it generalises Shannon's MI, holistically accounting for uncertainty within the Bayesian paradigm---measuring the amount of information a rational agent could extract from a random variable \emph{under partial knowledge of the true distribution}. 
In addition to the definition, our paper provides many useful theoretical results. 
For instance, we prove that conditioning does \emph{not} necessarily reduce 
Bayesian entropy and that Bayesian mutual information does \emph{not} obey
the data-processing inequality. 
We argue that these properties make our Bayesian framework ideal
for an analysis of learned representations.

In the empirical portion of our paper, we investigate both part-of-speech tagging and dependency arc labelling.
Moreover, our information-theoretic measure holistically captures a notion of ease of extraction, limiting the amount of data available to solve the task.
Intuitively, Bayesian MI shows that high dimensional representations, 
such as BERT, actually hurt performance in the very low-resource scenario, 
making less information available to a Bayesian agent than a simple categorical distribution.
This is because, when little data is available, these agents overfit to the evidence under weak priors.
In the high-resource scenario of English, ALBERT dominates the curves, making more information available than other contextualised embedders.
In short, Bayesian mutual information reconciles the probing literature with its frequently posed question: \textit{how much information can be extracted from these representations?}

\section{Background: Information Theory} \label{sec:information_theory}

Information theory \citep{shannon1948mathematical,cover1991elements}
provides us with a number of tools to analyse data and their associated probability distributions---among which are the entropy and mutual information. These are traditionally defined according to a ``true'' probability distribution, i.e. $p(x)$ or $p(x, y)$%
\footnote{We use uppercase letters to denote random variables ($X$, $Y$), lowercase for their instances ($x$, $y$), and calligraphic fonts for their domain space ($x \in \calX$, $y \in \calY$).}
which may not be known, but dictates the behaviour of random variables $X$ and $Y$. 
The atomic unit of information theory is the \defn{surprisal}, which is defined as follows:
\begin{equation} \label{eq:surprisal}
    \ent(X=x) = - \log p(x)
\end{equation}
Arguably, the most important information-theoretic definition is its expected value, termed \defn{entropy}:
\begin{equation}
    \ent(X) \defeq - \sum_{x \in \calX} p(x) \log p(x)
\end{equation}
Finally, another important concept is the \defn{mutual information} (MI) between two random variables
\begin{align}
    \MI(X; Y) &\defeq \ent(X) - \ent(X \mid Y)
\end{align}
Unfortunately, information theory has a few properties which do not conform to our intuitions about the mechanics of information in machine learning:

\begin{enumerate}[(i)]
\item \textbf{Data Does Not Add Information:}
The entropy is defined according to a source distribution $p(x)$. So, if multiple instances of $X$ are sampled i.i.d. from $p(x)$, access to a set $\vdn = \left\{x^{(1)}, \dots, x^{(N)}\right\}$
of such instances cannot provide any information, i.e. $\ent(X \mid \calDn=\vdn)=\ent(X)$. 
\item\textbf{Conditioning Reduces Entropy:}
Another basic result from information theory is that conditioning cannot increase entropy, only reduce it, i.e. $\ent(X \mid Y) \le \ent(X)$. This implies datapoints can never be misleading, which is not true in practice.\looseness=-1
\item\textbf{Data Processing Does Not Help:}
The data processing inequality states that processing some random variable with a function $f(\cdot)$ can never increase how informative it is, but only reduce its information content, i.e. $\MI(X; f(Y)) \le \MI(X; Y)$.
\end{enumerate}

\subsection{Background: Belief Entropy}

A related question that arises is how to estimate information in scenarios where the true distribution is not known. 
For instance, what is the surprisal of a learning agent with a \defn{belief} $\ptheta(x)$ who encounters an instance $x$?
The straightforward answer would be to use \cref{eq:surprisal}---nonetheless, this agent does not know the true distribution $p(x)$.
This agent's surprisal is usually taken according to its belief:
\begin{equation}\label{eq:belief_surprisal}
    \ent_b(X = \vx) = - \log \ptheta(\vx)
\end{equation}
Similarly, this agent's entropy has been historically defined exclusively according to this belief distribution \citep{gallistel2011memory}:
\begin{equation} \label{eq:belief_entropy}
    \ent_b(X) \defeq - \sum_{\vx \in \calX} \ptheta(\vx) \log \ptheta(\vx)
\end{equation}
We term this the \defn{belief-entropy}. We can further extend this to a \defn{belief mutual information}:%
\footnote{This definition is found in both the cognitive sciences \citep{gallistel2011memory,fan2014information,sayood2018information} as well as in active learning \citep{houlsby2011bayesian,kirsch2019batchbald}.}
\begin{align} \label{eq:belief_mi}
    \MI_b(X; Y) &\defeq \ent_b(X) - \ent_b(X \mid Y)
\end{align}
We note this definition is not grounded in the true distribution in any form.
In fact, about the belief mutual information, \citet{gallistel2011memory} state: ``the subjectivity that it implies is deeply unsettling [...] the amount of information actually communicated is not an objective function of the signal from which the subject obtained it''.

\section{A Bayesian Approach to Information}

The primary motivation for this paper is developing a series of tools that help us overcome the limitations of traditional information theory as applied to machine learning.
Specifically, probing representations requires a data-dependent information theory.
We thus formulate analogues of surprisal, entropy and MI in terms of Bayesian agents---using a framework heavily inspired by Bayesian experimental design \citep{lindley1956measure}.
We then prove this framework does not suffer the same infelicities as standard information theory in this context.

\subsection{An Agent-Based Information Theory}

Our discussions will focus on Bayesian agents, so we start by formally defining them.

\begin{defnt} \label{dfn:bayesian_agent}
A \defn{Bayesian agent} is a parameterised probability distribution $\ptheta(\vx \mid \vtheta)$ (or set of distributions) and a prior $\ptheta(\vtheta)$.\footnote{In the case that the Bayesian
agent has more than one distribution, we still only have a single prior without loss of generality. Indeed, separate priors for each distribution is a special case where the parameters are partitioned. In the case of $\vtheta = [\vphi;\vpsi]$, we could define $\ptheta(\vtheta)=\ptheta(\vphi)\cdot\ptheta(\vpsi)$.
}
Given data $\vdn = \{\vx^{(1)}, \ldots, \vx^{(N)}\}$, the \defn{Bayesian posterior}
over $\vtheta$ is%
\begin{equation}
    \ptheta(\vtheta \mid \vdn) \propto \prod_{n=1}^N\ptheta(\vx^{(n)} \mid \vtheta)\, \ptheta(\vtheta)
\end{equation}
Analogously, the \defn{Bayesian belief} is defined as the following posterior predictive distribution%
\begin{equation}
    \ptheta(\vx \mid \vdn) = \int \ptheta(\vx \mid \vtheta)\,\ptheta(\vtheta \mid \vdn)\,\mathrm{d}\vtheta
\end{equation}
\end{defnt}
\noindent Upon encountering an instance $\vx$, and after seeing a collection of data $\vdn$, this agent's \defn{posterior-predictive Bayesian surprisal} will be:%
\begin{equation}\label{eq:agent_surprisal}
    \enttheta(X = \vx \mid \calDn= \vdn) = - \log \ptheta(\vx\mid \vdn)
\end{equation}
where $\calDn$ is a data-valued random variable; for notational succinctness, we omit this random variable for the rest of the paper. 
We further define the \defn{posterior-predictive Bayesian entropy}:
\begin{align} \label{eq:agent_entropy}
    \enttheta(X \mid \vdn) = - \sum_{\vx \in \calX} p(\vx) \log \ptheta(\vx \mid \vdn)
\end{align}
As can be readily seen, the Bayesian entropy
is the expected value of the Bayesian surprisal---with this expectation taken over the true distribution.%
\footnote{We put this in contrast to \cref{eq:belief_entropy}---which takes this expectation over the belief itself---since the instances are in practice encountered with this true frequency. This distinction has been explicitly noted before, by e.g. \citet{bartlett1953statistical}. 
}
In this sense, the Bayesian entropy is a cross-entropy rather than a standard entropy.

\subsection{Bayesian Mutual Information}
Defining Bayesian mutual information within our framework requires
a bit more care. 
First, in contrast to surprisal and entropy, mutual information
is a functional of two random variables.
We will name the second random variable $Y$.
To talk about mutual information, we will consider a Bayesian agent with a
collection of at least two beliefs, e.g. $\{\ptheta(\vx), \ptheta(\vx \mid \vy)\}$.
The second belief is conditional, but otherwise follows
\cref{dfn:bayesian_agent}. 

\begin{defnt}
Given a collection of data $\vdn = \{(\vx^{(1)}, \vy^{(1)}), \ldots, (\vx^{(N)}, \vy^{(N)})\}$, and a Bayesian agent with a pair of beliefs $\ptheta(\vx \mid \vtheta)$ and $\ptheta(\vx \mid \vy, \vtheta)$ and a prior $\ptheta(\vtheta)$, the \defn{Bayesian mutual information} (Bayesian MI) is defined as
\begin{align}
    \MItheta(Y \rightarrow  &X \mid \vdn) \defeq \\
    &\enttheta(X \mid \vdn) - \enttheta(X \mid Y, \vdn) \nonumber
\end{align}
\end{defnt}

\noindent There is an important distinction between the Bayesian and Shannon MI---Bayesian MI decomposes as the difference between two cross-entropies, as opposed to two entropies.\footnote{Cross mutual information (XMI) has been used in several previous work such as \citep{pimentel-etal-2019-meaning,pimentel2020information,bugliarello-etal-2020-easier,mcallester2020formal,torroba2020intrinsic,fernandes-etal-2021-measuring,oconnor-andreas-2021-context}. In those works, though, it was usually interpreted as a computational approximation to the truth-MI (or to $\calV$-information \citep{xu2020theory}, which is discussed later in the paper). In this work, we highlight the Bayesian MI's (and XMI's) relevance as a generalisation of Shannon's MI.}

\subsection{An Illustrative Example} \label{sec:short_example}

\noindent For the sake of argument, we assume two independent categorical random variables $X$ and $Y$, both with $c$ classes and uniformly distributed.
\begin{align}
    p(x) = \frac{1}{c}, 
    ~~~ 
    p(x \mid y) = \frac{1}{c}
\end{align}
We further assume a Bayesian agent with two categorical beliefs $\{\ptheta(\vx)=\categorical(\vtheta), \ptheta(\vx \mid \vy)=\categorical(\vtheta+\vy)\}$---where $\vy$ is assumed to be encoded as a one hot vector---and a Dirichlet prior $\ptheta(\vtheta)=\dirichlet(\valpha)$ with concentration parameters $\valpha=\mathbf{1}$. 
Note that this (biased) agent believes that $\vy$ and $\vx$ are more likely than chance to share a class.
Given no data, or given $\vd_0$, this agent's prior predictive distributions are:\looseness=-1
\begin{align}
    \ptheta(x) = \frac{1}{c}, 
    \,\,\, 
    \ptheta(x \mid y) = \left\{ \begin{array}{cc}
        \frac{1}{c+1}, & \vx \neq \vy \\
        \frac{2}{c+1}, & \vx = \vy
    \end{array} \right.
\end{align}
In this example:
\begin{enumerate}[(i)]
    \item \textbf{Mutual Information.} We have $\MI(X; Y) = 0$ because $X$ and $Y$ are independent by construction.
    \item \textbf{Belief Mutual Information.} The belief-MI is positive, since the agent's uncertainty about $X$ is reduced by knowledge of $Y$---the prior predictive $\ptheta(x)$ is uniform, while the conditional distribution $\ptheta(x \mid y)$ is not, which reduces the belief-entropy. This means that $\ent_b(X) > \ent_b(X \mid Y)$, implying that $\MI_b(X; Y) > 0$.
    \item \textbf{Bayesian Mutual Information.} Finally, the Bayesian MI is negative---since $\vx$ is uniformly distributed, the unconditional Bayesian entropy is tight, i.e.  $\enttheta(X \mid \vd_0)=\ent(X)$, but the conditional one is not, i.e. $\enttheta(X \mid Y, \vd_0) > \ent(X \mid Y) = \ent(X)$. We thus have $\MItheta(X; Y) < 0$. This entails that, on this specific example, an agent's predictive power over $X$ is lower when given $Y$.
\end{enumerate}
This illustrates an important aspect of Bayesian MI: it is grounded on the true distribution.

\subsection{Theoretical Properties}

We now prove a few relevant theoretical properties about our framework. 
We show that Bayesian MI is symmetric if and only if the agent's beliefs respect Bayes' rule.
Then, we discuss why it does not respect the data-processing inequality, and its connection to mutual information and to $\mathcal{V}$-information \citep{xu2020theory}.\looseness=-1

\subsubsection{When is Bayesian MI Symmetric?}
It is a well known result that Shannon's MI is symmetric, i.e.
\begin{align}
    \MI(X; Y) &= \ent(X) - \ent(X \mid Y) \\
    &= \ent(Y) - \ent(Y \mid X) = \MI(Y; X)\nonumber
\end{align}
This means that the knowledge one can extract from random variable $Y$ about $X$ is the same as the knowledge one can extract from $X$ about $Y$. 
This is not true in general for Bayesian MI; 
as we will show, information-theoretic symmetry and Bayes rule are tightly related.
As such, we consider in this section a Bayesian agent with a set of beliefs $\{\ptheta(\vx), \ptheta(\vx \mid \vy), \ptheta(\vy), \ptheta(\vy \mid \vx)\}$. 
We call the agent \defn{consistent} if it respects Bayes' rule, i.e.
\begin{equation}
    \ptheta(\vx \mid \vy) = \frac{\ptheta(\vy \mid \vx)\,\ptheta(\vx)}{\ptheta(\vy)}
\end{equation}
the following theorem characterises when we have symmetry.
\begin{thm}\label{thm:symmetry}
An agent's Bayesian mutual information is symmetric, i.e.
\begin{equation}
    \MItheta(X \rightarrow Y \mid \vdn) =  \MItheta(Y \rightarrow X \mid \vdn) 
\end{equation}
for all distributions $p(x, y)$
if and only if the Bayesian agent is consistent.
\end{thm}
\begin{proof}
See \cref{app:proof_symmetry}.
\end{proof}

\subsubsection{No Data-Processing Inequality} \label{sec:function}

\noindent Another classical result from information theory is the data processing inequality. This theorem states that processing a random variable can never add information, only reduce it
\begin{equation} \label{eq:data_inequality}
    \MI(X; Y) \ge \MI(X; f(Y))
\end{equation}
Although theoretically sound, this theorem is very unintuitive from a practical perspective---effectively, processing noisy data can make it more useful.
In fact, representation learning is a subfield of machine learning devoted precisely to finding functions which can extract more ``informative'' representations from some input.
One such example is BERT \citep{devlin-etal-2019-bert}, a large pre-trained language model which produces contextualised representations from sentential inputs.
These representations provably contain the exact same information about any task as the original sentence \citep{pimentel2020information}---in practice, though, they are much more useful for downstream models.

The data processing inequality does not hold for Bayesian information, making it a more intuitive information-theoretic measure for probing;
pre-trained representation extraction functions can increase MI from a Bayesian agent perspective.
\begin{thm} \label{thm:data_processing_inequality}
The data processing inequality does not hold for Bayesian information, i.e.
\begin{equation} \label{eq:data_observer_inequality}
    \MItheta(Y \rightarrow X \mid \vdn) ~?~ \MItheta(f(Y) \rightarrow X \mid \vdn)
\end{equation}
\end{thm}
\begin{proof}
See \cref{app:proof_data_inequality}.
\end{proof}

\subsubsection{Relation to Mutual Information}

The relationship between Bayesian mutual information and Shannon MI is relevant for our discussion.
As mentioned in the introduction, \citeauthor{shannon1948mathematical} was concerned with the limits of communication when he defined his measure. 
We now put forward an intuitive theorem about Bayesian information; it is upper-bounded by the true MI under a weak assumption about the agent's beliefs.

\begin{thm} \label{thm:lower_bound_mi}
Assuming the agent's belief $\ptheta(\vx \mid \vdn)$ has a smaller Kullback--Leibler (KL) divergence when compared to the true $p(\vx)$ than the marginal of its beliefs over $\vy$, i.e. 
\begin{align}
    \KL\bigg(p(x) &\mid\mid \ptheta(x \mid \vdn) \bigg) \le \\
    &\KL\left(p(x) \mid\mid \sum_{\vy \in \calY}\ptheta(x \mid \vy, \vdn)\, p(\vy) \right) \nonumber
\end{align}
We show
\begin{equation}
    \MItheta(Y \rightarrow X \mid \vdn) \le \MI(X; Y)
\end{equation}
\end{thm}
\begin{proof}
See \cref{app:proof_agent_upperbound}.
\end{proof}

\noindent In other words, the information any agent can extract from a random variable $Y$ about another variable $X$ is upper-bounded by the true MI.
We now define a well-formed belief, which we will use to analyse the Bayesian MI's convergence:
\begin{defnt}
We say the belief of a Bayesian agent is \defn{well-formed} if and only if the true distribution is a possible belief, i.e.
\begin{align}
    \exists\,\vtheta: \quad \ptheta(\vtheta) > 0~~\mathrm{and}~~p(\vx) = \ptheta(\vx \mid \vtheta)
\end{align}
\end{defnt}

\noindent Given this definition, we prove %
the Bayesian mutual information converges to the true MI under well-defined conditions.
\begin{thm} \label{thm:convergence_mi}
If we assume a Bayesian agent's set of beliefs and prior are well-formed and meet the conditions
of Bernstein--von Mises Theorem
\citep[pg. 339,][]{bickel2015mathematical}.%
\footnote{In the case where $\vtheta$ is discrete and finite, the only requirement is $\ptheta(\vtheta) > 0$, for all values of $\vtheta$ \citep{freedman1963asymptotic}.}
Then, 
\begin{equation}
    \lim_{N \rightarrow \infty} \MItheta(Y \rightarrow X \mid \vdn) = \MI(X; Y)
\end{equation}
\end{thm}
\begin{proof}
See \cref{app:proof_convergence_mi}.
\end{proof}

\subsubsection{Relation to Variational Information}

Variational ($\calV$-) information \citep{xu2020theory} is a recent generalisation
of mutual information.
It extends MI to the case where a fixed family of distributions is considered; in which the true distribution may or not be.
\begin{defnt}
Suppose %
random variable $X$ is distributed according to $p(x)$. 
Let $\calV$ be a variational family of distributions.
Then, $\calV$-entropy is defined
\begin{equation}
    \ent_{\calV}(X) \defeq \inf_{q \in \calV} -\sum_{x \in \calX} p(x) \log q(x)
\end{equation}
and $\calV$-information is defined as 
\begin{equation}
    \MI_{\calV}(Y \rightarrow X) \defeq \ent_\calV (X) -  \ent_\calV (X \mid Y)
\end{equation}
\end{defnt}
\noindent Unlike our Bayesian mutual information, $\calV$-information is not a
data-dependent measure, i.e. $\ent_{\calV}(X \mid \calDn) = \ent_{\calV}(X)$. Thus, it does not meet
our desiderata. 
However, we can prove a straightforward relationship
between the Bayesian and $\calV$ informations,
which we state below.\looseness=-1
\begin{thm} \label{thm:convergence_vmi}
Assume a Bayesian agent's beliefs and prior meet the conditions of \citet{kleijn2012bernstein}, who extend the Bernstein--von Mises Theorem to beliefs which are not well-formed.
Further, let $\calV = \{\ptheta(\cdot \mid \vtheta) \mid \ptheta(\vtheta) > 0 \}$. Then,%
\begin{equation}
    \lim_{N \rightarrow \infty} \MItheta(Y \rightarrow X \mid \vdn) = \MIcalV(Y \rightarrow X)
\end{equation}
\end{thm}
\begin{proof}
See \cref{app:proof_convergence_vmi}.
\end{proof}

\section{A Framework for Incremental Probing}
The proposed Bayesian framework for information allows us to take into account
the amount of data we have for probing.
Crucially, previous work \cite{pimentel2020information} 
failed to adequately account for the observation of data.
 In doing so, they only analysed the limiting behaviour of information, 
under which the probing enterprise is not fully sensible---given unlimited data and computation, there is no point in using pre-trained functions.
Indeed, the higher-level motivation of this work is to find an information-theoretic framework which serves machine learning, 
and under which the goal of probing \emph{is} inherently clear.
To that end, we propose a relatively simple experimental design. 
We compute Bayesian mutual information, which is a function
of the amount of data, to create several learning curves.\looseness=-1

\paragraph{Notation.}
We define a sentence-level random variable $S$, with instances $\vs$ taken from $V^*$,
the Kleene closure of a potentially infinite vocabulary $V$. 
We further define a representation-valued random variable $R$ and a task-valued random variable $T$, each with instances $\vr \in \R^d$ and $\vt \in \calT$, where $\calT$ is the set of possible values for the analysed task (e.g. the set of parts of speech in a language).

\subsection{Probes as Bayesian Agents}

The overall trend in NLP is to train supervised probabilistic models on task-specific data.
We believe probabilistic probes should analogously be modelled this way---leading to results compatible with our empirical intuitions. 
We thus define a \defn{probe agent} as a Bayesian agent with the pair of beliefs $\{\ptheta(\vt \mid \vtheta), \ptheta(\vt \mid \vr, \vtheta)\}$ and a prior $\ptheta(\vtheta)$.
Any prior $\ptheta(\vtheta)$ could be chosen for our probing agents.
Nonetheless we have no \emph{a priori} knowledge of how the representations should impact our prediction task.
As such, our priors are such that the initial distributions $\ptheta(\vt \mid \vd_0)$ and $\ptheta(\vt \mid \vr, \vd_0)$ are identical.
A logical conclusion, is that the prior Bayesian MI should be zero:\looseness=-1
\begin{align}
    \MItheta(R \rightarrow T \mid \vd_0) = 0
\end{align}
On the opposite extreme---i.e. given unlimited data---a well-formed belief will likely converge to the true distribution, yielding the same results as by \citet{pimentel2020information}. Complementarily, an ill-formed belief will converge to the $\calV$-information:
\begin{align}
    \lim_{N \rightarrow \infty} \MItheta(R \rightarrow T \mid \vdn) 
    &= \MIcalV(R \rightarrow T) \\
    &\approx \MI(S; T) \nonumber
\end{align}

The novelty of our framework lies in the explicit analysis of information under finite data. 
Bayesian agents are used here to measure a notion of information directly related to \defn{ease of extraction}---i.e. how much information could be extracted from the representations by a na\"{i}ve agent with no \emph{a priori} knowledge about the task itself. 
In other words, we ask the question: \emph{given a specific dataset $\vdn$, how much information do the representations yield about this task?} 
This value is only a subset of the true MI, being upper-bounded by it.

\paragraph{Why Bayesian MI and not Bayesian entropy?}
We focus our analysis on the amount of information a Bayesian agent can extract from the representations about the task. 
However, we could as easily analyse the Bayesian entropy instead.
We believe, though, that the Bayesian MI is an inherently more intuitive value than the entropy.
This is because mutual information puts the Bayesian entropy in perspective to a trivial baseline---how much uncertainty would there be \emph{without the representations}. 
Furthermore, it has a much more interpretable value: with no data its value is zero, while at the limit it converges to the true mutual information.
In this paper, we are concerned with its trajectory, i.e., how fast does the Bayesian MI go up?%
\looseness=-1

\newcommand{\plottask}{pos_tag}
\newcommand{\plotsize}{0.25\textwidth}
\newcommand{\trimcommand}{}

\begin{figure*}
    \centering
\tagmcbegin{tag=Figure,alttext={Plots in English, Basque, Marathi and Turkish of the Bayesian MI in part of speech tagging.}}
    \begin{subfigure}[t]{\plotsize}
    \includegraphics[width=\columnwidth]{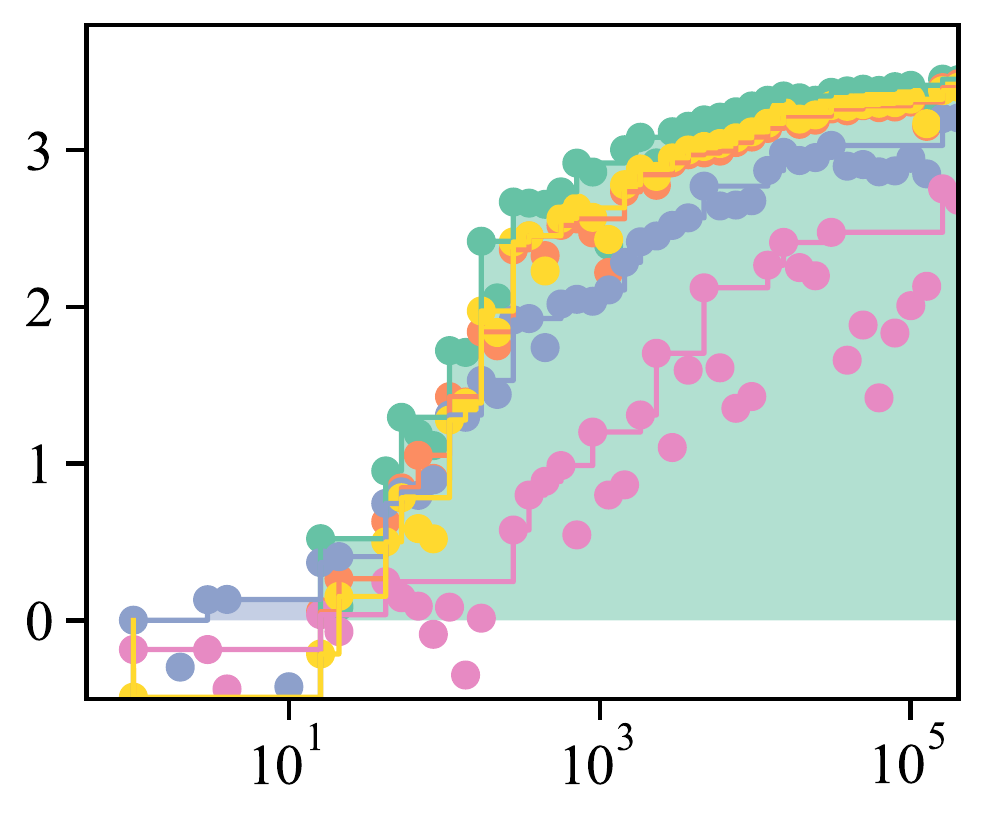}
    \end{subfigure}%
    ~ 
    \begin{subfigure}[t]{\plotsize}
    \includegraphics[width=\columnwidth]{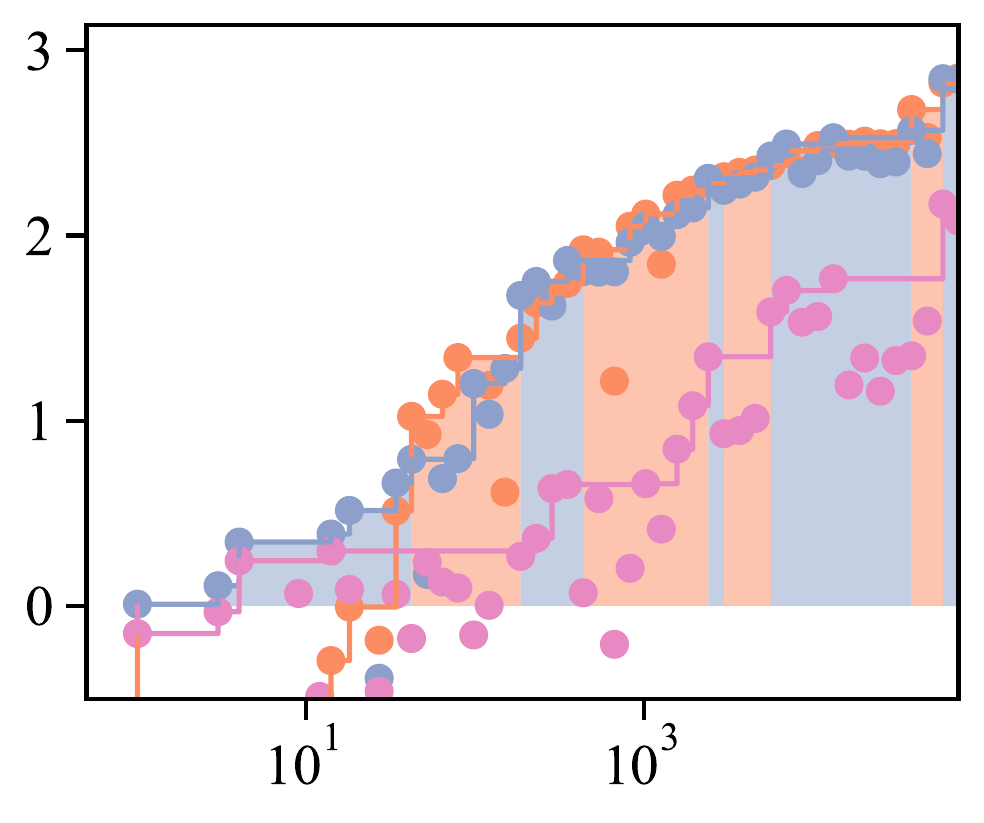}
    \end{subfigure}%
    ~ 
    \begin{subfigure}[t]{\plotsize}
    \includegraphics[width=\columnwidth]{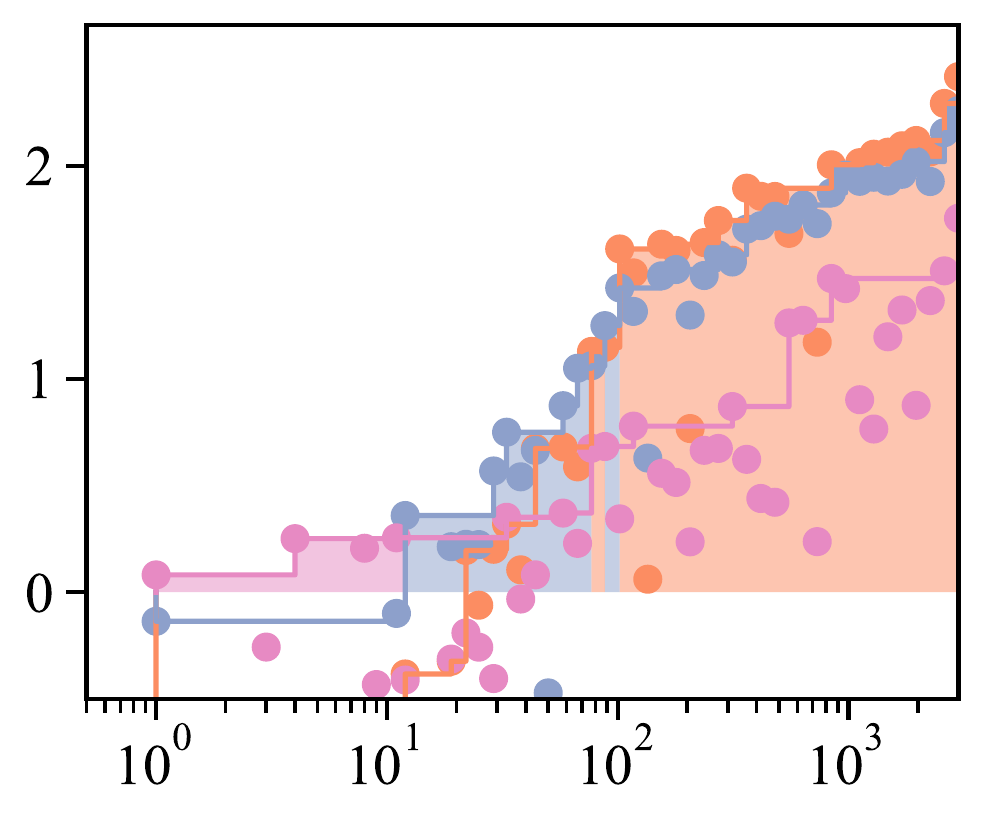}
    \end{subfigure}%
    ~ 
    \begin{subfigure}[t]{\plotsize}
    \includegraphics[width=\columnwidth]{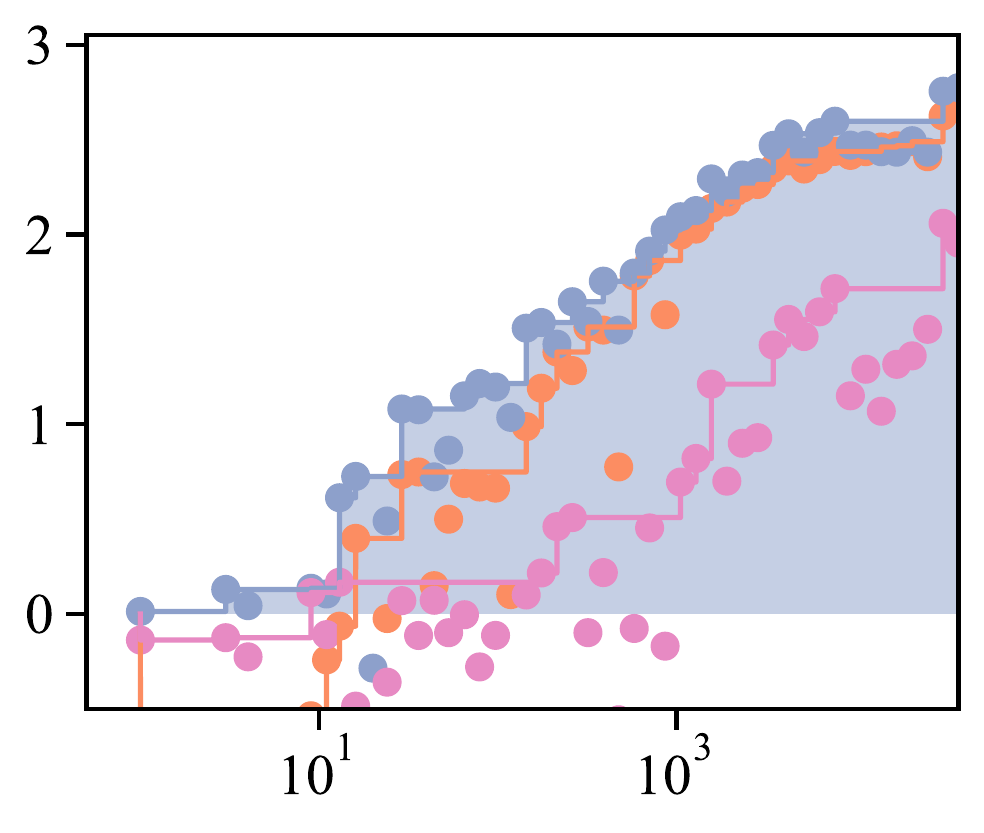}
    \end{subfigure}
    \vspace{-21pt}
    \caption{
    Bayesian MI (bits; $y$-axis) vs number of data examples ($x$-axis) in part of speech tagging on  (left) English (center-left) Basque, (center-right) Marathi, and (right) Turkish.
    {\color{codec0} \textbf{ALBERT}}, {\color{codec5} \textbf{RoBERTa}}, {\color{codec1} \textbf{BERT}}, {\color{codec2} \textbf{fastText}}, {\color{codec3} \textbf{Random}}
    }
    \label{fig:pos_tag}
\tagmcend
\end{figure*}

\subsection{Ease of Extraction and Previous Work}

Generally speaking, the goal of probing is to test if a set of contextual representations encodes a certain linguistic property \citep[\textit{inter alia}]{adi2016fine,belinkov-etal-2017-evaluating,tenney2018what,liu2021probing}.
Most work in this field claims that, when performing this analysis, we should prefer simple models as probes \citep{alain2016understanding,hewitt-liang-2019-designing,voita2020information}.
This is inline with \citeauthor{pimentel2020information}'s results: using a complex probe (complex enough to ensure it is well-formed) with infinite data, we would estimate $\MI(S; T)$---a value which does not meaningfully inform us about the representations themselves. Defining model complexity, though, is not trivial \citep[for a longer discussion see][]{pimentel-etal-2020-pareto}. For this reason, many works limit themselves to studying only linearly encoded information \citep[e.g.][]{alain2016understanding,hewitt-manning-2019-structural,maudslay-etal-2020-tale} or a subset of neurons at a time
\citep[e.g.][]{torroba2020intrinsic,mu2020compositional,durrani2020analyzing}.
However, restricting our analysis this way seems arbitrary.

A few recent papers have tried to deal with probe complexity in a more nuanced way. 
\citet{hewitt-liang-2019-designing} argue for the use of selectivity to control for probe complexity.
\citet{voita2020information} and \citet{whitney2020evaluating} use, respectively, minimum description length (MDL) and surplus description length (SDL) to measure the size (in bits) of the probe model.
\citet{pimentel-etal-2020-pareto} argues probe complexity and accuracy should be seen as a Pareto trade-off, and propose new metrics to measure probe complexity.
All of these papers define ease of extraction in terms of properties of the probe, e.g., its complexity and size.

We argue here for an opposing view of ease of extraction: Instead of focusing on the complexity of the probes, we should define it according to the \defn{complexity of the task}.
We further operationalise this complexity in a very specific way: how much information a Bayesian agent can extract from the representations, given \emph{limited knowledge} about the task itself---where this limited knowledge is enforced by the size of the observed dataset $\vdn$.%
\footnote{As we show later in the paper, this background knowledge about the task can also be formally defined as a Bayesian mutual information, i.e. the information the observed data provides about the model parameters $\MItheta(\calDn \rightarrow \Theta)$.
}
With this in mind, we evaluate the Bayesian MI learning curves.
In this regard, our analysis is similar to the learning curves used by \citet{talmor2019olmpics} and the complexity--accuracy trade-offs from \citet{pimentel-etal-2020-pareto}.

\section{Experiments and Results\footnote{Our code is available in \url{https://www.github.com/rycolab/bayesian-mi}.}}

\subsection{Data and Representations}

We focus on part-of-speech (POS) tagging and dependency-arc labelling in our experiments.
With this in mind, we make use of the universal dependencies \citep[UD 2.6;][]{ud-2.6}; analysing
the treebanks of four typologically diverse languages, namely: Basque, English, Marathi, and Turkish.
As our object of analysis, we look at the contextual representations from ALBERT \cite{lan2019albert}, RoBERTa \cite{liu2019roberta} and BERT \cite{devlin-etal-2019-bert},%
\footnote{We use the pre-trained models made available by the transformers library \citep{wolf2019hugging}.}
using as a baseline the non-contextual fastText \cite{fasttext} and random embeddings. Random embeddings are initialised at the type level and kept fixed during experiments.\looseness=-1

\vspace{-1pt}
\subsection{Probe}
\vspace{-1pt}

Our experiments focus on Bayesian agents with multi-layer perceptron (MLP) beliefs:
\begin{align}
    &\ptheta(\vt \mid \vr, \vtheta) = \softmax(\mathrm{MLP}(\vr; \vphi)) \label{eq:probe_conditional} \\
    &\ptheta(\vt \mid \vtheta) = \categorical(\vpsi) = \frac{\psi_t}{\sum_{t'=1}^{|\calT|} \psi_{t'}} \label{eq:probe_unconditional}
\end{align}
where
$\vphi$ are the MLP parameters, $\vpsi \in \mathbb{N}^{|\calT|}$ is a count vector and $\vtheta = [\vphi; \vpsi]$ are the agent's parameters. 
This agent has a Gaussian prior over parameters $\vphi$ (with zero mean and standard deviation $\sigma=10$),
and a Dirichlet distribution prior over $\vpsi$ (with concentration parameter $\alpha=1$).\looseness=-1

As previously discussed, 
the Gaussian and Dirichlet priors on the parameters will cause these models to initially place a uniform distribution on the output classes---as such,
they will have an initial Bayesian MI of zero.
We then expose the probe agent to increasingly larger sets of data from the task.
Unfortunately, the posterior %
of \cref{eq:probe_conditional} has no closed form solution, so we approximate it with the maximum-a-posteriori probability
$\ptheta(\vt \mid \vr, \vtheta^{*})$,
where $\vtheta^{*} = \argmax_{\vtheta \in \Theta} \ptheta(\vtheta \mid \vd_n)$. 
We obtain this MAP estimate using the gradient descent method AdamW \citep{loshchilov2019decoupled} with a cross-entropy loss and L2 norm regularisation.%
\footnote{L2 weight decay regularisation is equivalent to a Gaussian prior on the parameter space \citep[pg. 350,][]{bishop1995neural}.}
The posterior predictive belief of \cref{eq:probe_unconditional} has a closed-form solution\footnote{This posterior predictive distribution is equivalent to Laplace smoothing \citep{jeffreys1939theory,robert2009harold}.}
\begin{align}
    \ptheta(\vt \mid \vdn) 
    &= \frac{\mathrm{count}(\vdn, t) + 1}{N + |\calT|}
\end{align}
where $\mathrm{count}(\vdn, t)$ is the number of  observed instances of class $t$.

\renewcommand{\plottask}{dep_label}
\renewcommand{\plotsize}{0.25\textwidth}

\begin{figure*}
    \centering
\tagmcbegin{tag=Figure,alttext={Plots in English, Basque, Marathi and Turkish of the Bayesian MI in dependency arc labelling.}}
    \begin{subfigure}[t]{\plotsize}
    \includegraphics[width=\columnwidth]{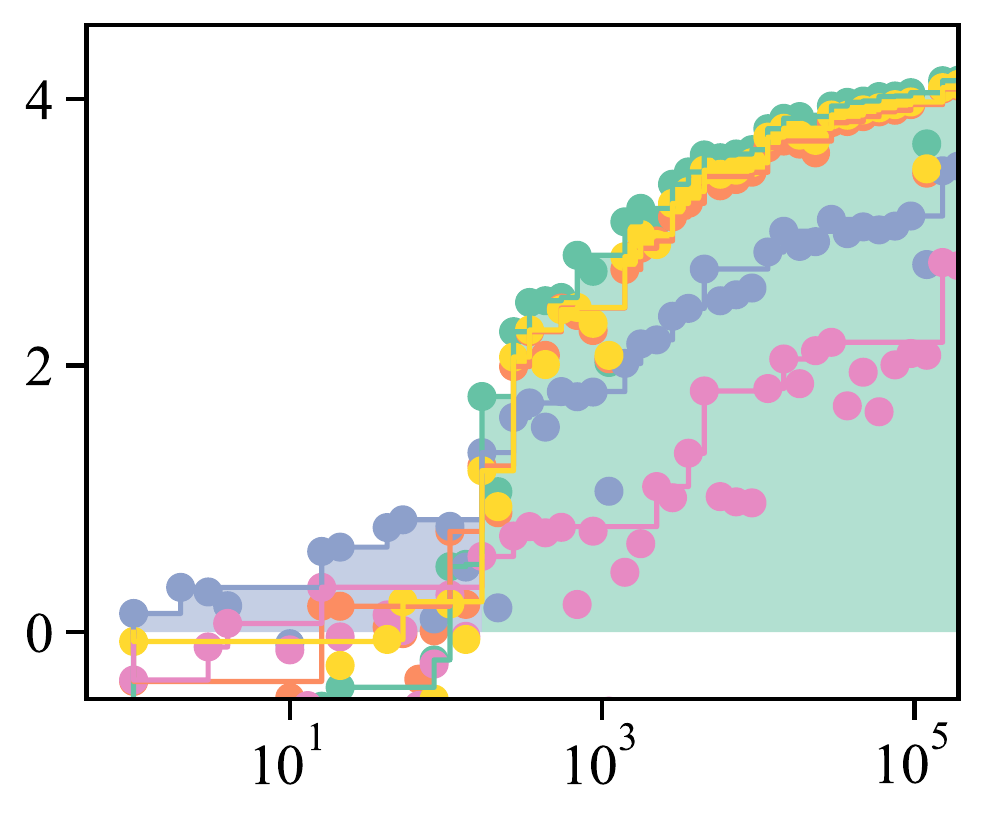}
    \end{subfigure}%
    ~ 
    \begin{subfigure}[t]{\plotsize}
    \includegraphics[width=\columnwidth]{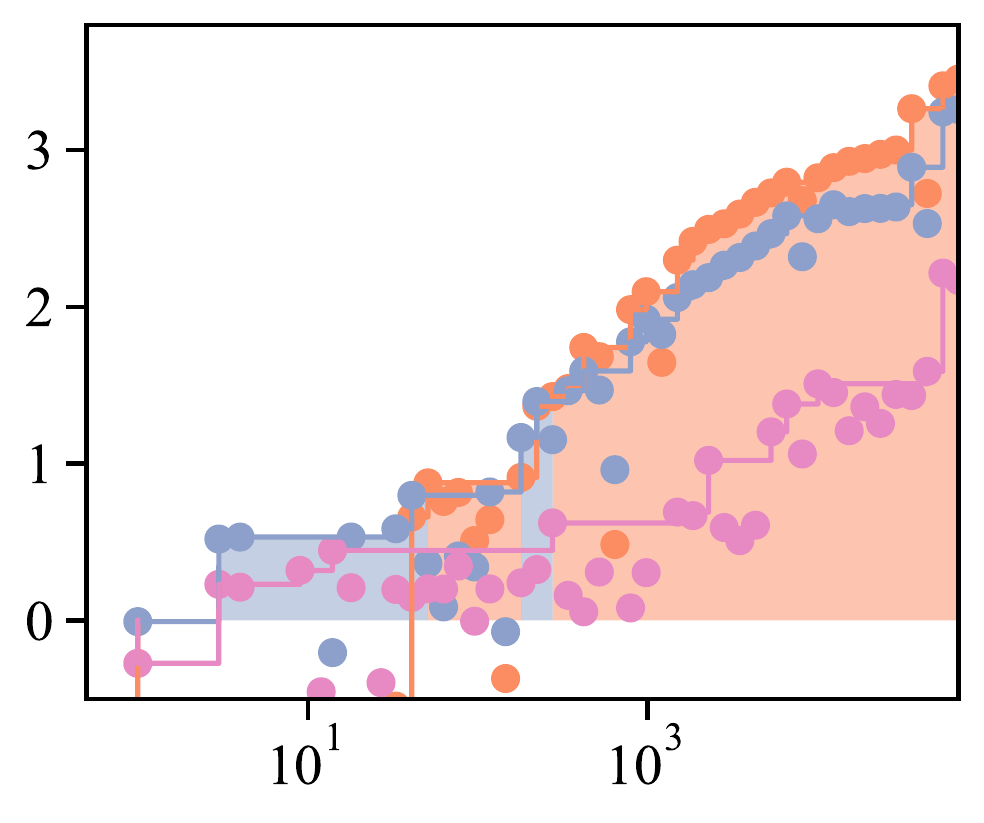}
    \end{subfigure}%
    ~ 
    \begin{subfigure}[t]{\plotsize}
    \includegraphics[width=\columnwidth]{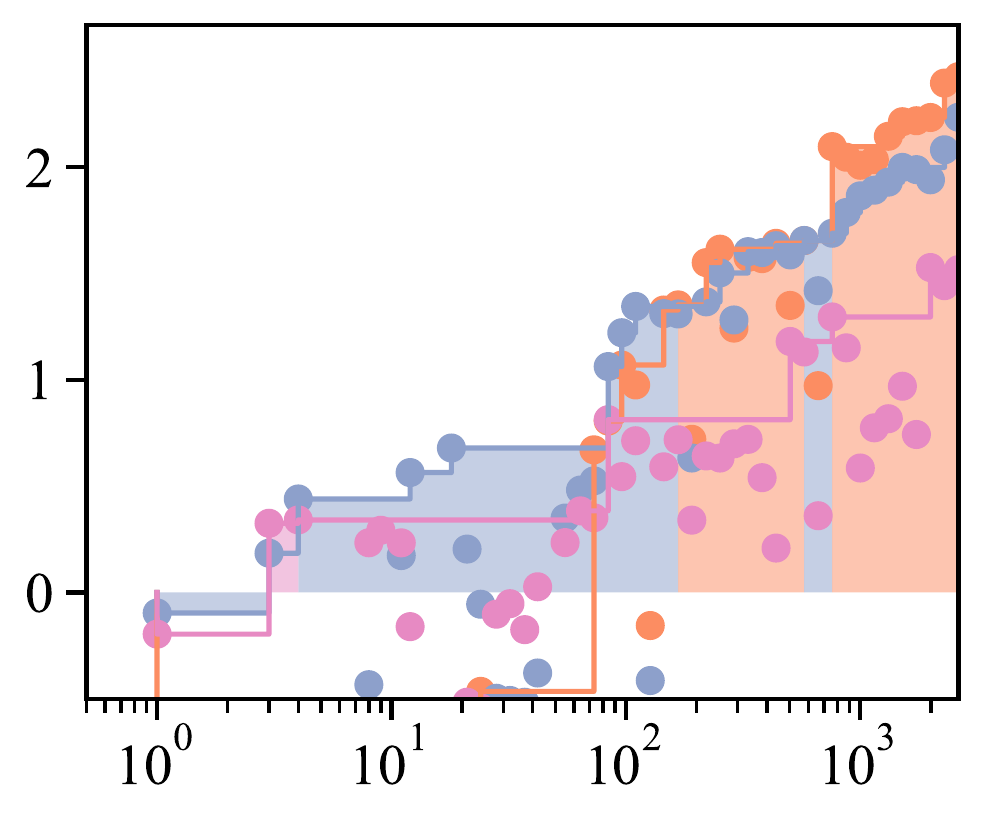}
    \end{subfigure}%
    ~ 
    \begin{subfigure}[t]{\plotsize}
    \includegraphics[width=\columnwidth]{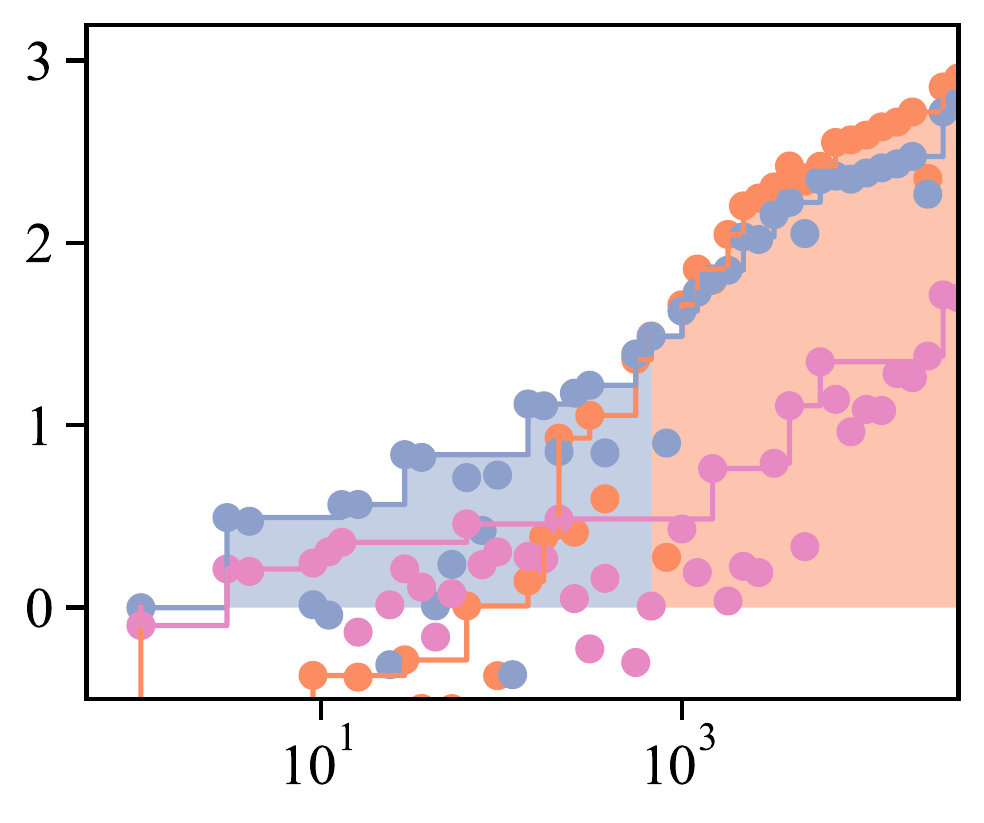}
    \end{subfigure}
    \vspace{-23pt}
    \caption{Bayesian MI (bits; $y$-axis) vs number of examples ($x$-axis) in dependency arc labelling on (left) English (center-left) Basque (center-right) Marathi, and (right) Turkish.
    {\color{codec0} \textbf{ALBERT}}, {\color{codec5} \textbf{RoBERTa}}, {\color{codec1} \textbf{BERT}}, {\color{codec2} \textbf{fastText}}, {\color{codec3} \textbf{Random}}
    }
    \label{fig:dep_arc}
\tagmcend
\end{figure*}

For both analysed tasks, we run 50 experiments with log-linearly increasing data sizes, from 1 instance to the whole language's treebank. 
For each of these individual experiments, we sample an MLP probe configuration. This probe will have 0, 1, or 2 layers---where 0 layers means a linear probe---dropout between 0 and 0.5, and hidden size from 32 to 1024 (log distributed). We then use the same architecture to train a probe for each of our analysed representations, plotting their Pareto curves.\looseness=-1

\subsection{Discussion}

\Cref{fig:pos_tag} presents pareto curves for part-of-speech tagging. These curves convey a few interesting results. The first is the intuitive fact that information is much harder to extract with random embeddings, although with enough training data their results slowly converge to near the fastText ones---this can be seen most clearly in English. 
This matches our theoretical framework: the true mutual information between the target task and either fastText or random embeddings is the same, thus, if our beliefs are well-formed, the Bayesian MI should converge to this value, although with different speeds. 
The second result is that ALBERT makes information more easily extractable than either BERT or RoBERTa in English, and that multilingual BERT is roughly equally as informative as fastText under the finite data scenarios of the other analysed languages. 
Finally, the last result goes against one of the claims of \citet{pimentel-etal-2020-pareto}, who in light of their flat Pareto curves for POS tagging claimed that we needed harder tasks for probing. 
One only needs harder tasks if their measure of complexity is not nuanced enough---as we see, even POS tagging is hard under the low-resource scenarios presented in our learning curves.\looseness=-1

\Cref{fig:dep_arc} presents results for dependency arc labelling. These learning curves also present interesting trends. While the POS tagging curves seem to be on the verge of convergence for English, Basque and Turkish, this is not the case for dependency arc labelling. 
This implies that, as expected, dependency arc labelling is either an inherently harder task, or that the representations encode the necessary information in a harder to extract manner.
These results, also highlight the importance of an information-theoretic measure being able to capture negative information---as evinced in \cref{fig:dep_arc_english}. 
For the low-data scenario, the BERToid models hurt performance, as opposed to helping. 
This is because high-dimensional representations, together with a weak prior, allow the agent to easily overfit to the little presented evidence.
On the other hand, fastText does not present the same problem, having a positive Bayesian MI even in a low-data setting.

\begin{figure}
    \centering
\tagmcbegin{tag=Figure,alttext={Polyfits of the Bayesian MI results in dependency arc labelling in English.}}
    \includegraphics[width=\columnwidth]{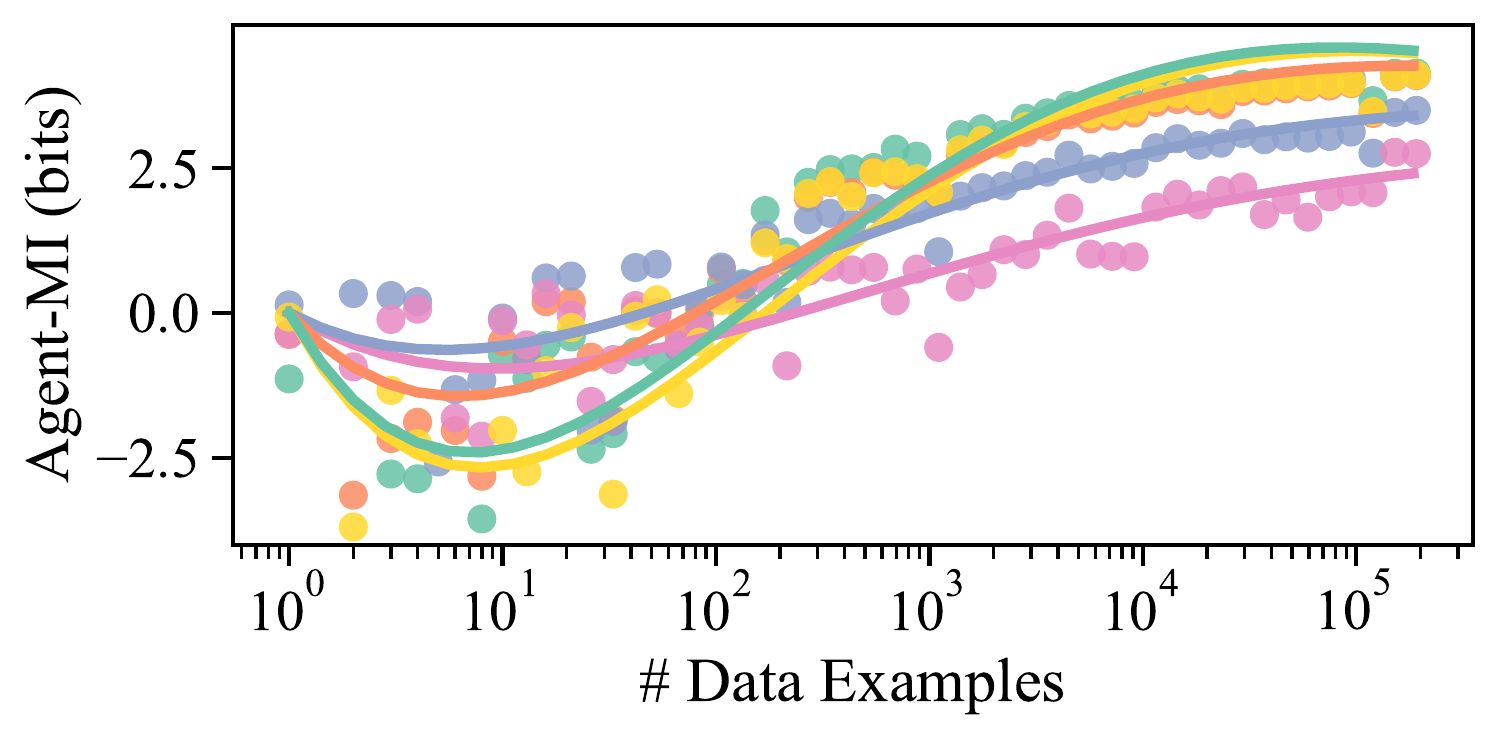}
    \caption{Dependency arc labelling polyfit on English.
    }
    \label{fig:dep_arc_english}
\tagmcend
\end{figure}

\section{An Intuitive Decomposition}

We now present some
basic results about our framework which, although not strictly necessary for the present study, help motivate it.
They also serve as a justification for our choice of cross-entropy when formalising Bayesian entropy.
With this in mind, we analyse information from the perspective of a fully Bayesian agent with a well-formed belief.%
\footnote{We make the same analysis from the perspective of an agent with an ill-formed belief in \cref{sec:ill_formed_decomposition}}
A classic decomposition of the cross-entropy is the following:\looseness=-1%
\begin{equation} \label{eq:cros_kl}
    \enttheta(T) = \ent(T) + \KL(p \mid\mid \ptheta)
\end{equation}
We posit a new interpretation for this equality.

\begin{thm}\label{thm:bayesian_decomposition}
Let $\Theta$ be a parameter-valued random variable.
The entropy of a consistent Bayesian agent with well-formed beliefs decomposes as
\begin{equation}
    \enttheta(T \mid \vdn) = \underbrace{\ent(T)}_{\text{entropy}} + \underbrace{\MItheta(T \rightarrow \Theta \mid \vdn)}_{\text{information about distribution}} \nonumber
\end{equation}
\end{thm}
\begin{proof}
See \cref{app:proof_decomposition}.
\end{proof}
\noindent In other words, the cross-entropy is composed of the sum between the entropy itself---i.e. the ``true'' information the data source provides, or its inherent uncertainty---and how much information the data provides about its distribution itself.

\paragraph{Relation to SDL.}
The minimum description length \citep[MDL;][]{voita2020information} is a probing metric defined as $\enttheta(\calDn)$.
In its online coding interpretation, it is rewritten as \citep{rissanen1978modeling,blier2018description}: $\enttheta(\calDn) = \sum_{n=1}^N \enttheta(X \mid \calD_{n-1})$---where the cross-entropy of each element $X$ in $\calDn$ is computed incrementally because the parameter $\vtheta$ (which would make them independent) is unknown.
The surplus description length \citep[SDL;][]{whitney2020evaluating}
is defined as the difference between a dataset's cross-entropy and its entropy: $\enttheta(\calDn) - \ent(\calDn)$. 
Using \cref{thm:bayesian_decomposition}, we derive a new interpretation for SDL:
\begin{equation} \label{eq:info_dataset}
 \MItheta(\calDn \rightarrow \Theta) = \enttheta(\calDn) - \ent(\calDn)
\end{equation}
where we use prior predictive distributions, as opposed to posterior predictive ones.
From this equation, we find that SDL is the information a dataset gives a Bayesian agent about its model parameters.

While closely related to one another, the Bayesian MI, MDL and SDL converge to different values in the limit of infinite dataset sizes:
\begin{align}
    &\lim_{N \rightarrow \infty} \underbrace{\enttheta(\calDn)}_{\text{MDL}} \rightarrow \infty \\
    &\lim_{N \rightarrow \infty} \underbrace{\MItheta(\calDn \rightarrow \Theta)}_{\text{SDL}} \rightarrow \infty \\
    &\lim_{N \rightarrow \infty} \underbrace{\MItheta(Y \rightarrow X \mid \vdn)}_{\text{Bayesian MI}} \rightarrow \MI(X; Y)
\end{align}
It is easy to see that MDL goes to infinity as the dataset size grows---$\enttheta(\calDn)$ grows at least linearly with the data size.
The reasons behind SDL also exploding as the data increases are less straightforward, though, but become clear from its Bayesian MI interpretation.
If the parameter space is continuous, and if the Bayesian belief converges at the limit of infinite data (as per \cref{thm:convergence_mi}), the Bayesian mutual information in \cref{eq:info_dataset} will naturally go to infinity.\footnote{This is a byproduct of the properties of differential entropies (the entropy of continuous random variables). As the distribution $\ptheta(\vtheta \mid \vdn)$ converges to a Dirac delta distribution centred on the optimal parameters, which has an entropy of negative infinity, this Bayesian MI goes to positive infinite.%
\begin{align}
    \lim_{N \rightarrow \infty} \MItheta(\calDn \rightarrow \Theta) &= \lim_{N \rightarrow \infty} \left( \enttheta(\Theta) - \enttheta(\Theta \mid \vdn) \right) \nonumber \\
    &= \enttheta(\Theta) - (-\infty) \nonumber \\
    &= \infty
\end{align}
}
We thus argue that Bayesian mutual information is a better measure for probing than either MDL or SDL; 
although all are sensitive to the observed dataset size, Bayesian MI is the only that does not diverge as this size grows.

\section{Conclusion}

In this paper we proposed an information-theoretic framework to analyse mutual information from the perspective of a Bayesian agent; we term this Bayesian mutual information.
This framework has intuitive properties (at least from a machine learning perspective), which traditional information theory does not, for example: data can be informative, processing can help, and information can hurt.
In the experimental portion of our paper, we use Bayesian mutual information to probe representations for both part-of-speech tagging and dependency arc labelling.
We show that ALBERT is the most informative of the analysed representations in English;
and high dimensional representations can provide negative information on low data scenarios.

\section*{Acknowledgements}

We thank Adina Williams, Vincent Fortuin, Alex Immer, Lucas Torroba Hennigen and M\'{a}rio Alvim for providing feedback in various stages of this paper. We also thank the anonymous reviewers for their valuable feedback in improving this paper.

\section*{Ethical Considerations}

The authors foresee no ethical concerns with the research presented in this paper.

\bibliography{agents}
\bibliographystyle{acl_natbib}

\appendix
\onecolumn

\section{Ill-formed Beliefs Loose Information} \label{sec:ill_formed_decomposition}

For the sake of argument, we now assume an agent with an ill-defined belief $\ptheta(\vt \mid \vtheta)$ and a prior $\ptheta(\vtheta)$. We will show that such Bayesian agents loose information, meaning that they will not obtain as much information about their optimal parameters as if they had a well-formed belief.

\begin{thm} \label{thm:bayes_wrong_lose}
Assume $\vtheta^*$ are the optimal parameters for a Bayesian agent with ill-formed, but consistent beliefs. The information this agent will receive about its optimal parameters is 
\begin{equation}
    \MItheta(T \rightarrow \Theta = \vtheta^* \mid \vdn) < \enttheta(T \mid \vdn) - \ent(T)
\end{equation}
\end{thm}
\begin{proof}
This proof follows from the Bayesian MI definition, from this Bayesian agent having consistent beliefs, 
and from the fact that the cross-entropy is an upper-bound to the entropy, with equality only when both probability distributions are the same---which by definition is not possible $\ptheta(\vt \mid \vtheta^*) \neq p(\vt)$
\begin{subequations}
\begin{align}
\MItheta(T \rightarrow \Theta = \vtheta^* \mid \vdn) &= 
\MItheta(\Theta  = \vtheta^* \rightarrow T \mid \vdn) \qquad \text{\color{gray} \small (symmetry due to belief consistency)}\\
    &= \enttheta(T \mid \vdn) - \enttheta(T \mid \Theta=\vtheta^*, \vdn) \\
    &= \enttheta(T \mid \vdn) - \enttheta(T \mid \Theta=\vtheta^*) \\
    &< \enttheta(T \mid \vdn) - \ent(T) \qquad \text{\color{gray} \small (strict inequality due to ill-formed beliefs)}
\end{align}
\end{subequations}
\end{proof}

\section{Measures of Information} \label{sec:measures_of_info}

Several other measures of information have been proposed, among them are the $H$ entropy \citep{degroot1962uncertainty}, the R\'{e}nyi entropy \citep{renyi1961measures,lenzi2000statistical}, Bayes vulnerability \citep{alvim2019science}, and the Determinantal Mutual Information \citep[DMI;][]{kong2020dominantly}. None of these take an agent's belief into consideration, and so our analysis is orthogonal to them.
The work most similar to ours, in this respect, is  \citeposs{clarkson2005belief} investigation of how belief impacts information leakage---and its extension, by \citet{hamadou2010reconciling}, to the R\'enyi min-entropy.
Importantly, the results obtained by \citeauthor{clarkson2005belief} 
can be similarly derived using our framework.

\section{A Note on Empirical Limitations} \label{sec:limitations}

Estimating the true MI between two random variables is known to be a hard problem for which several methods have been proposed \citep[for a detailed review, see ][]{mcallester2020formal}---estimating the Bayesian MI may be equally challenging. 
Given knowledge of $\ptheta(\cdot)$ and having access to samples from $p(\cdot)$, the Bayesian MI can be trivially estimated using the Bayesian surprisal's sample mean. 
On the other hand, in a setting such as active learning, where one (by definition) does not have access to the true distribution $p(y \mid x)$---only to the belief---the best approximation to the Bayesian MI may indeed be the belief-MI (used by \citealt{houlsby2011bayesian}) or the Bayesian surprise (used by \citealt{storck1995reinforcement} and \citealt{itti2006bayesian,itti2009bayesian}).
Finally, approximating the Bayesian MI in the cognitive sciences may be an even harder problem than estimating the true MI, since it would require approximating both the belief $\ptheta(\cdot)$ of a specific agent and the true distribution $p(\cdot)$ of an event.

\section{Proof of Symmetric Bayesian Mutual Information, \Cref{thm:symmetry}} \label{app:proof_symmetry}

\paragraph{\Cref{thm:symmetry}.}
\textit{An agent's Bayesian mutual information is symmetric, i.e.}
\begin{equation}
    \MItheta(X \rightarrow Y \mid \vdn) =  \MItheta(Y \rightarrow X \mid \vdn)
\end{equation}
\textit{for all distributions $p(x, y)$ if and only if the Bayesian agent is consistent.}
\begin{proof}
We will first prove that if the Bayesian MI is symmetric for all true distributions $p(x, y)$, then the Bayesian agent is consistent (the \emph{if} case).
We then prove the inverse proposition (the \emph{only if} case), completing this \textit{if and only if} theorem's proof.

\paragraph{$\Rightarrow$} 
We show this, by relying on specific distributions where $p(x, y)$ is deterministic, putting all probability mass in a single point, i.e. $p(x_0, y_0) = 1$.
\begin{subequations}
\begin{align}
    \MItheta(X \rightarrow Y) &= \MItheta(Y \rightarrow X) \\
    \sum_{x \in \calX}\sum_{y \in \calY} p(x, y) \log \frac{\ptheta(y \mid x)}{\ptheta(y)} 
    &= \sum_{x \in \calX}\sum_{y \in \calY} p(x, y) \log \frac{\ptheta(x \mid y)}{\ptheta(x)} \\
    p(x_0, y_0) \log \frac{\ptheta(y_0 \mid x_0)}{\ptheta(y_0)} 
    &= p(x_0, y_0) \log \frac{\ptheta(x_0 \mid y_0)}{\ptheta(x_0)} \\
    \frac{\ptheta(y_0 \mid x_0)}{\ptheta(y_0)} 
    &= \frac{\ptheta(x_0 \mid y_0)}{\ptheta(x_0)}
\end{align}
\end{subequations}
As we can show this same result for any value of $x \in \calX$ and $y \in \calY$, we conclude these agents must have consistent beliefs, i.e.
\begin{equation}
    \forall \vx \in \calX, \vy \in \calY:\qquad \frac{p(y \mid x)}{p(y)} = \frac{p(x \mid y)}{p(x)}
\end{equation}

\paragraph{$\Leftarrow$} 
The Bayes rule can be written as
\begin{equation}
    \frac{p(y \mid x)}{p(y)} = \frac{p(x \mid y)}{p(x)}
\end{equation}
We now show that all consistent agents  will have symmetric MI
\begin{subequations}
\begin{align}
    \MItheta(X \rightarrow Y) &= \sum_{x \in \calX}\sum_{y \in \calY} p(x, y) \log \frac{\ptheta(y \mid x)}{\ptheta(y)} \\ 
    &\stackrel{(1)}{=} \sum_{x \in \calX}\sum_{y \in \calY} p(x, y) \log \frac{\ptheta(x \mid y)}{\ptheta(x)} \\
    &= \MItheta(Y \rightarrow X)
\end{align}
\end{subequations}
where equality (1) relies on Bayes rule.

\end{proof}

\section{Proof of No Data Processing Inequality, \Cref{thm:data_processing_inequality}} \label{app:proof_data_inequality}

\paragraph{\Cref{thm:data_processing_inequality}.}
\textit{The data processing inequality does not hold for Bayesian information, i.e.}
\begin{equation}
    \MItheta(Y \rightarrow X \mid \vdn) ~~?~~ \MItheta(f(Y) \rightarrow X \mid \vdn)
\end{equation}

\begin{proof}
We prove this theorem with a counter example.
Let $Y$ be a Poisson distributed random variable with unknown mean, i.e. $p(y) = \poisson(\hat{\theta})$, where $\hat{\theta}$ is this distributions true mean. We further define a second random variable $Z$, as:
\begin{align}
    f(y) = y - \hat{\theta}, \qquad p(z \mid y) = \mathbbm{1}\{z = f(y)\}
\end{align}
$Z$ is, thus, a deterministic function of $Y$, where the function $f(y)$ mean-centres random variable $Y$.
Finally, we also define $X$ as a Bernoulli distributed random variable:
\begin{align}
    g(z) = \frac{1}{1+ e^{-z}}, 
    \qquad 
    p(x \mid z) = \left\{\begin{array}{cc}
        g(z) & x=1 \\
        1-g(z) & x=0
    \end{array} \right.
\end{align}
where $g(z)$ is a sigmoid function. We can further define the distribution
\begin{align}
    p(x \mid y) = \left\{\begin{array}{cc}
        \left(g \circ f\right)(y) & x=1 \\
        1- \left(g \circ f\right)(y) & x=0
    \end{array} \right.
\end{align}
From these distributions, we see that $\ent(Z \mid Y) = 0$, as $f(y)$ is a deterministic function. 
Further, in this specific example $\ent(X \mid Z) = \ent(X \mid Y)$---this can be proved from the fact that $f(y)$ is a bijective function, and that the data processing inequality is tight for bijections
\begin{align}
    \ent(X \mid Y) \le \ent(X \mid f(Y)) = \ent(X \mid Z) \le \ent(X \mid f^{-1}(Z)) = \ent(X \mid Y)
\end{align}
We now define a Bayesian agent, which correctly knows the relationship between $Y$, $Z$ and $X$, i.e. with well-formed beliefs $\ptheta(x \mid y, \theta)$ and $\ptheta(x \mid z, \theta)$, and with a prior $\ptheta(\theta)$---this agent does not know the true value of parameter $\theta$ though.
For this Bayesian agent
\begin{align}
    \enttheta(X \mid Y, \vdn) 
    = \sum_{x \in \calX}\sum_{y \in \calY} p(x, y) \log \int \ptheta(x \mid y, \theta) \ptheta(\theta \mid \vdn)  \mathrm{d}\theta 
\end{align}
When given $Z=f(Y)$, however, this agent does not need to know $\theta$, since the data is already mean-centred (there are no unknown parameters in $\ptheta(x \mid z)$).
This Bayesian agent's conditional entropy given $Z$ is
\begin{subequations}
\begin{align}
    \enttheta(X \mid Z, \vdn) &= \sum_{x \in \calX}\sum_{z \in \calZ} p(x, z) \log \ptheta(x \mid z, \vdn) \\
    & = \sum_{x \in \calX}\sum_{z \in \calZ} p(x, z) \log p(x \mid z, \vdn) \\
    & = \sum_{x \in \calX}\sum_{z \in \calZ} p(x, z) \log p(x \mid z) \\
    & = \ent(X \mid Z)
\end{align}
\end{subequations}
This concludes the example that a deterministic (mean-centring) function can help this Bayesian agent.
\begin{subequations}
\begin{align}
    \MItheta(Y \rightarrow X \mid \vdn) &= \enttheta(X \mid \vdn) - \enttheta(X \mid Y, \vdn) \\
    &\stackrel{(1)}{\le} \enttheta(X \mid \vdn) - \ent(X \mid Y) \\
    &= \enttheta(X \mid \vdn) - \ent(X \mid Z) \\
    &= \enttheta(X \mid \vdn) - \enttheta(X \mid Z, \vdn) \\
    &= \MItheta(Z  \rightarrow X \mid \vdn) \\
    &= \MItheta(f(Y)  \rightarrow X \mid \vdn)
\end{align}
\end{subequations}
where (1) becomes a strict inequality if the belief $\ptheta(\theta) \neq \delta(\theta - \hat{\theta})$, i.e. if the prior does not place all probability mass in the true parameters $\hat{\theta}$.
\end{proof}

\section{Proof of Bayesian MI is Upper-bounded by the True MI, \Cref{thm:lower_bound_mi}} \label{app:proof_agent_upperbound}

\paragraph{\Cref{thm:lower_bound_mi}.}
\textit{Assuming the agent's belief $\ptheta(\vx \mid \vdn)$ is tighter than the marginal of its beliefs over $\vy$, i.e. than $\sum_{\vy \in \calY} \ptheta(\vx \mid \vy, \vdn) p(\vy)$. We show}
\begin{equation}
    \MItheta(Y \rightarrow X \mid \vdn) \le \MI(X; Y)
\end{equation}
\begin{proof}
We start by noting that the difference between the true MI, and its Bayesian counterpart is equal to the difference between two KL-divergences
\looseness=-1
\begin{subequations}
\begin{align}
    \MI(X; Y) - \MItheta(Y \rightarrow X \mid \vdn) 
    &= \ent(X) - \ent(X \mid Y) - \enttheta(X \mid \vdn) + \enttheta(X \mid Y, \vdn) \\
    &=\KL(p(x \mid y) \mid\mid \ptheta(x \mid y, \vdn)) - \KL(p(x) \mid\mid \ptheta(x \mid \vdn)) 
\end{align}
\end{subequations}
To prove this theorem, we need to show that one KL-divergence is smaller than the other, i.e.
\begin{equation} \label{eq:diff_kl}
\KL(p(x) \mid\mid \ptheta(x \mid \vdn)) \le \KL(p(x \mid y) \mid\mid \ptheta(x \mid y, \vdn)) 
\end{equation}
We can show this with a bit of algebraic manipulation and an assumption about our Bayesian agent
\begin{subequations}
\begin{align}
    \KL(p(x \mid y) \mid\mid \ptheta(x \mid y, \vdn)) 
    &= \sum_{x \in \calX} p(x) \sum_{y \in \calY} p(y \mid x) \log \frac{p(x \mid y)}{\ptheta(x \mid y, \vdn)}  \\
    &= \sum_{x \in \calX} p(x) \sum_{y \in \calY} p(y \mid x) \log \frac{p(x \mid y) p(y)}{\ptheta(x \mid y, \vdn)p(y)} \\
    &\stackrel{(1)}{\ge} \sum_{x \in \calX} p(x) \left(\sum_{y \in \calY} p(y \mid x) \right) \log \frac{\sum_{y \in \calY} p(x \mid y) p(y)}{\sum_{y \in \calY} \ptheta(x \mid y, \vdn)p(y)} \\
    &= \sum_{x \in \calX} p(x) \cdot 1 \cdot \log \frac{p(x)}{\sum_{y \in \calY} \ptheta(x \mid y, \vdn)p(y)} \\
    &\stackrel{(2)}{\ge} \sum_{x \in \calX} p(x) \log \frac{p(x)}{\ptheta(x \mid \vdn)} \\
    &= \KL(p(x) \mid\mid \ptheta(x \mid \vdn))
\end{align}
\end{subequations}
In this equations, (1) relies on the log sum inequality, while (2) assumes the following inequality
\begin{subequations}
\begin{align}
    \enttheta(X) &= - \sum_{x \in \calX} p(x) \log \ptheta(x \mid \vdn) \\
    &\le - \sum_{x \in \calX} p(x) \log \sum_{y \in \calY} \ptheta(x \mid y, \vdn) p(y) 
\end{align}
\end{subequations}
This is equivalent to our assumption that this agent's estimate of $\ptheta(x \mid \vdn)$ is tighter than if the agent marginalised its beliefs over $y$.
While not necessarily true, in practice, if $X$ is discrete and has a small cardinality $|\calX|$, a simple Laplace smoothed estimate of $\ptheta(x \mid \vdn)$ is likely to result in this inequality.
One could instead assume an agent which uses a Monte Carlo sampling approximation for estimating $\ptheta(x \mid \vdn)$ from $\ptheta(x \mid y, \vdn)$. This would switch the inequality (2) for an approximation, and result in an expected lower bound instead
\begin{equation}
    \MItheta(Y \rightarrow X \mid \vdn) \lesssim \MI(X; Y)
\end{equation}
\end{proof}

\section{Proof of the Convergence to Mutual Information, \Cref{thm:convergence_mi}} \label{app:proof_convergence_mi}

\paragraph{\Cref{thm:convergence_mi}}
\textit{
If we assume a Bayesian agent's set of beliefs and prior are well-formed and meet the conditions
of Bernstein--von Mises Theorem
\citep[pg. 339,][]{bickel2015mathematical}.
Then,}
\begin{equation}
    \lim_{N \rightarrow \infty} \MItheta(Y \rightarrow X \mid \vdn) = \MI(X; Y)
\end{equation}
\begin{proof}
The Bernstein--von Mises Theorem only applies to well-formed beliefs, i.e. beliefs which can model the true probability distribution---a condition which is satisfied by our assumptions to this theorem.
By this theorem---and under a number of other specified conditions, e.g. absolute continuity of the prior in a neighbourhood around $\truetheta$ and continuous positive density at $\truetheta$ (see pg. 141 in \citealt{van2000asymptotic} for the full set of conditions)---we have
\begin{equation}
    p(x) = \lim_{N\rightarrow \infty} \ptheta(x \mid \vdn) 
\end{equation}
Now, we apply the continuous mapping theorem to analyse the convergence of the Bayesian entropy
\begin{subequations}
\begin{align}
    \lim_{N \rightarrow \infty} \enttheta(X \mid \vdn) &= - \lim_{N \rightarrow \infty} \sum_{\vx \in \calX} p(\vx) \log \ptheta(\vx \mid \vdn) \\
    &\overset{(1)}{=} - \sum_{\vx \in \calX} p(\vx) \log \left( \lim_{N \rightarrow \infty} \ptheta(\vx \mid \vdn)\right)\\
    &= - \sum_{\vx \in \calX} p(\vx) \log p(\vx) \\
    &= \ent(X)
\end{align}
\end{subequations}
where (1) relies on the continuous mapping theorem. A similar convergence applies to $\enttheta(X \mid Y, \vdn)$. Finally, we can complete the proof
\begin{subequations}
\begin{align}
    \lim_{N \rightarrow \infty} \MItheta(Y \rightarrow X \mid \vdn) 
    &= \lim_{N \rightarrow \infty} \left( \enttheta(X \mid \vdn) - \enttheta(X \mid Y, \vdn) \right) \\
    &= \ent(X) - \ent(X \mid Y) \\
    &= \MI(X; Y)
\end{align}
\end{subequations}

\end{proof}

\section{Proof of the Convergence to $\calV$-information, \Cref{thm:convergence_vmi}} \label{app:proof_convergence_vmi}

\paragraph{\Cref{thm:convergence_vmi}}
\textit{
Assume a Bayesian agent's beliefs and prior meet the conditions of \citet{kleijn2012bernstein}, who extend the Bernstein--von Mises Theorem to beliefs which are not well-formed.
Further, let $\calV = \{\ptheta(\cdot \mid \vtheta) \mid \ptheta(\vtheta) > 0 \}$. Then,}%
\begin{equation}
    \lim_{N \rightarrow \infty} \MItheta(Y \rightarrow X \mid \vdn) = \MIcalV(Y \rightarrow X)
\end{equation}
\begin{proof}
\citet{kleijn2012bernstein} extend the Bernstein--von Mises Theorem to ill-formed beliefs, showing that, under specific conditions for the Bayesian belief and priors, the predictive posterior distribution converges to
\begin{equation}
    \lim_{N\rightarrow \infty} \ptheta(x \mid \vdn) = \ptheta(x \mid \vtheta^{*})
\end{equation}
where $\vtheta^*$ is a unique set of parameters which minimises the KL-divergence between $\ptheta(\vx \mid \vtheta)$ and the true distribution $p(\vx)$, i.e.
\begin{equation}
    \ptheta(x \mid \vtheta^{*}) = \mathrm{arg}\inf_{q \in \calV} \sum_{x \in \mathcal{X}} p(x) \log \frac{1}{q(x)}
\end{equation}
Given this convergence property, we can finish the proof similarly to the one for the well-formed belief:
\begin{subequations}
\begin{align}
    \lim_{N \rightarrow \infty} \enttheta(X \mid \vdn) 
    &= \lim_{N \rightarrow \infty} \sum_{\vx \in \calX} p(\vx) \log \frac{1}{\ptheta(\vx \mid \vdn)} \\
    &\overset{(1)}{=} \sum_{\vx \in \calX} p(\vx) \log \left( \lim_{N \rightarrow \infty} \frac{1}{\ptheta(\vx \mid \vdn)} \right) \\
    &= \sum_{\vx \in \calX} p(\vx) \log \frac{1}{\ptheta(\vx \mid \vtheta^{*})} \\
    &= \ent_\calV(X)
\end{align}
\end{subequations}
where $\calV$ is defined as $\{p(x \mid \vtheta) \mid \ptheta(\vtheta) > 0\}$, and (1) relies on the continuous mapping theorem.
We now conclude this proof:
\begin{subequations}
\begin{align}
    \lim_{N \rightarrow \infty} \MItheta(Y \rightarrow X \mid \vdn) 
    &= \lim_{N \rightarrow \infty} \left( \enttheta(X \mid \vdn) - \enttheta(X \mid Y, \vdn) \right) \\
    &= \ent_\calV(X) - \ent_\calV(X \mid Y) \\
    &= \MIcalV(Y \rightarrow X)
\end{align}
\end{subequations}

\end{proof}

\section{Proof of the Intuitive Decomposition, \Cref{thm:bayesian_decomposition}} \label{app:proof_decomposition}

\paragraph{\Cref{thm:bayesian_decomposition}}
\textit{Let $\Theta$ be a parameter-valued random variable.
The entropy of a consistent Bayesian agent with well-formed beliefs decomposes as}
\begin{equation} \label{eq:cross_mi}
    \enttheta(T \mid \vdn) = \underbrace{\ent(T)}_{\text{entropy}} + \underbrace{\MItheta(T \rightarrow \Theta \mid \vdn)}_{\text{information about distribution}}
\end{equation}

\begin{proof}
Note that under the Bayesian MI only the information about the true model parameters, i.e. $\truetheta$, matters%
\begin{subequations}
\begin{align}
    \MItheta(X \rightarrow \Theta \mid \vdn) 
    &= \sum_{x \in \calX} \int\,p(x, \vtheta) \log \frac{\ptheta(\vtheta \mid x, \vdn)}{\ptheta(\vtheta \mid \vdn)} \mathrm{d}\vtheta \\
     &\stackrel{(1)}{=} \sum_{x \in \calX} \int\,p(x)\,\delta(\vtheta - \truetheta) \log \frac{\ptheta(\vtheta \mid x, \vdn)}{\ptheta(\vtheta \mid \vdn)} \mathrm{d}\vtheta \\
    &= \sum_{x \in \calX} \,p(x) \log \frac{\ptheta(\truetheta \mid x, \vdn)}{\ptheta(\truetheta \mid \vdn)} \\
    &= \MItheta(X \rightarrow \Theta = \truetheta \mid \vdn) 
\end{align}
\end{subequations}
where (1) relies on the fact that the true $p(\vtheta)$ places all probability mass on the value $\truetheta$.
Using this result, we can show the Bayesian mutual information in \cref{eq:cross_mi} is the same as the KL-divergence.
\begin{subequations}
\begin{align}
    \MItheta(X \rightarrow \Theta \mid \vdn) &= \MItheta(X \rightarrow \Theta=\truetheta \mid \vdn) \\
    &\stackrel{(2)}{=} \MItheta(\Theta = \truetheta \rightarrow X \mid \vdn) \\
    &= \enttheta(X \mid \vdn) - \enttheta(X \mid \Theta=\truetheta, \vdn) \\
    &= \enttheta(X \mid \vdn) - \ent(X) \\
    &= \KL(p(x) \mid\mid \ptheta(x \mid \vdn)) 
\end{align}
\end{subequations}
where (2) relies on the assumption that this agent's beliefs are consistent, and by definition $\ent(X)=\enttheta(X \mid \Theta=\truetheta, \vdn)$.
\end{proof}

\end{document}